\documentclass{article}

\usepackage[dvipsnames]{xcolor}

\usepackage{microtype}
\usepackage{graphicx}
\usepackage{subfigure}
\usepackage{booktabs}       % professional-quality tables
\usepackage{amsfonts}       % blackboard math symbols
\usepackage{nicefrac}       % compact symbols for 1/2, etc.
\usepackage{enumitem}

\usepackage{hyperref}

\usepackage[accepted]{icml2024}

\usepackage{amsmath}
\usepackage{amssymb}
\usepackage{mathtools}
\usepackage{amsthm}
\usepackage{algorithm2e}
\RestyleAlgo{ruled}

\usepackage[capitalize,noabbrev]{cleveref}

\theoremstyle{plain}
\newtheorem{theorem}{Theorem}[section]

\newtheorem{lemma}[theorem]{Lemma}

\theoremstyle{definition}

\newtheorem{assumption}[theorem]{Assumption}
\theoremstyle{remark}

\newcommand{\E}{\mathbb{E}}
\newcommand{\R}{\mathbb{R}}
\newcommand{\N}{\mathbb{N}}
\newcommand{\prob}{\mathbb{P}}
\newcommand{\indep}{\perp \!\!\! \perp}

\usepackage{xspace}

\usepackage{bbm}

\makeatletter
\def\maketag@@@#1{\hbox{\m@th\normalfont\normalsize#1}}
\makeatother

\icmltitlerunning{Meta-Learners for Partially-Identified Treatment Effects Across Multiple Environments}

\begin{document}

\twocolumn[
\icmltitle{Meta-Learners for Partially-Identified Treatment Effects Across Multiple Environments}



\icmlsetsymbol{equal}{*}

\begin{icmlauthorlist}
\icmlauthor{Jonas Schweisthal}{equal,lmu,mcml}
\icmlauthor{Dennis Frauen}{equal,lmu,mcml}
\icmlauthor{Mihaela van der Schaar}{cam}
\icmlauthor{Stefan Feuerriegel}{lmu,mcml}
\end{icmlauthorlist}

\icmlaffiliation{lmu}{LMU Munich, Germany}
\icmlaffiliation{mcml}{Munich Center for Machine Learning (MCML),  Germany}
\icmlaffiliation{cam}{University of Cambridge, UK}

\icmlcorrespondingauthor{Jonas Schweisthal}{jonas.schweisthal@lmu.de}
\icmlcorrespondingauthor{Dennis Frauen}{frauen@lmu.de}

\icmlkeywords{Machine Learning, ICML, causal inference, partial identification, unobserved confounding, instrumental variables}

\vskip 0.3in
]



\printAffiliationsAndNotice{\icmlEqualContribution} % otherwise use the standard text.

\begin{abstract}
Estimating the conditional average treatment effect (CATE) from observational data is relevant for many applications such as personalized medicine. Here, we focus on the widespread setting where the observational data come from multiple environments, such as different hospitals, physicians, or countries. Furthermore, we allow for violations of standard causal assumptions, namely, overlap within the environments and unconfoundedness. To this end, we move away from point identification and focus on \emph{partial identification}. Specifically, we show that current assumptions from the literature on multiple environments allow us to interpret the environment as an instrumental variable (IV). This allows us to adapt bounds from the IV literature for partial identification of CATE by leveraging treatment assignment mechanisms across environments. Then, we propose different model-agnostic learners (so-called meta-learners) to estimate the bounds that can be used in combination with arbitrary machine learning models. We further demonstrate the effectiveness of our meta-learners across various experiments using both simulated and real-world data. Finally, we discuss the applicability of our meta-learners to partial identification in instrumental variable settings, such as randomized controlled trials with non-compliance. 
\end{abstract}

\section{Introduction}
\label{sec:intro}

Estimating conditional average treatment effects (CATEs) from observational data is relevant for
many applications. Examples include medicine \cite{Glass.2013, feuerriegel2024causal}, economics \cite{Angrist.1990, kuzmanovic2024causal}, or marketing \cite{Varian.2016}. For example, medical professionals are interested in leveraging electronic health records to personalize care by understanding the estimated CATE of treatments. 

In this paper, we are interested in the CATE in a setting where we have access to observational data, that are collected from \emph{multiple environments}. Furthermore, we allow for \emph{violations of standard causal assumptions}, namely, overlap within the environments and unconfoundedness. Our setting is particularly relevant for medical applications for two reasons. 

\emph{First}, CATE estimation in practice often involves settings where the observational data come from  \emph{multiple environments} \cite{Shi.2021}. Common examples in medicine are settings where patient data come from multiple hospitals, multiple physicians, or multiple countries \citep[e.g.,][]{Huang.2022}. As a result, each environment is characterized by unique patient demographics (e.g., a specialized hospital may have more severe cases of disease) and/or unique treatment policies (e.g., the default treatment option may vary across countries). 

\begin{figure*}[t]
  \centering
  \includegraphics[width=0.95\linewidth]{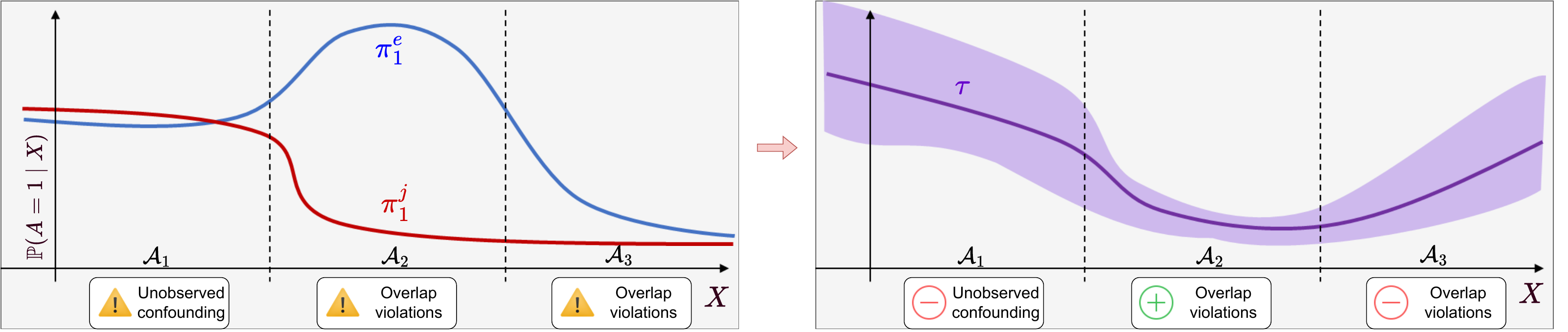}
  \caption{Intuition for our bounds. \emph{Left:} Two propensity scores \textcolor{blue}{$\pi^e_1(x) = \prob(A^e = 1 \mid X^e = x)$} and \textcolor{BrickRed}{$\pi^j_1(x) = \prob(A^j = 1 \mid X^j = x)$} corresponding to different environments \textcolor{blue}{$e$} and \textcolor{BrickRed}{$j$} are plotted over observed confounders $X$. Large values of \textcolor{blue}{$\pi^e_1(x)$} and \textcolor{BrickRed}{$\pi^j_1(x)$} correspond to a high probability of receiving treatment and vice versa. \emph{Right:} CATE $\tau(x)$ together with bounds depending on violations of overlap and unconfoundedness. In region $\mathcal{A}_1$, no overlap violations occur, leading to wide bounds for the CATE $\tau$ due to potential unobserved confounding. In region $\mathcal{A}_2$, overlap violations occur at \emph{opposite ends} across environments, leading to tight bounds for $\tau$. In region $\mathcal{A}_3$, overlap violations occur on the same end across environments, leading to wide bounds for $\tau$ due to a lack of data for treated individuals.}
  \label{fig:intro}
\end{figure*}

\emph{Second}, identification of the CATE from observational data is challenging and typically requires strong \emph{assumptions} such as (i)~\emph{overlap (within environments)} and (ii)~\emph{unconfoundedness} \cite{Wager.2018}. (i)~Overlap ensures that each individual has a positive probability of receiving any treatment \cite{DAmour.2021} within each environment. (ii)~Unconfoundedness implies that all potential confounders that influence both the treatment assignment and the outcome of interest are contained within the observational data. Under violations of these assumptions, the CATE is generally unidentifiable and can not be estimated consistently from observational data \cite{Pearl.2009}. Notwithstanding, \emph{violations} are likely to occur in settings with multiple environments. (i)~When patient characteristics and/or treatment policies vary across environments, then it is likely that certain patient characteristics and/or certain treatments are rarely or even never observed in some environments (e.g., certain treatments are only administered by specialized physicians). As such, overlap is often violated. (ii)~Confounders such as socioeconomic status are often not recorded in medical studies \cite{Adler.2015}. As such, unconfoundedness is often violated. Consequently, standard CATE estimators may be biased and thus unreliable, motivating the need for methods to relax standard assumptions.

To address the challenges of the setting above, we move away from \emph{point} identification and focus on \emph{partial} identification. Partial identification means that one relaxes assumptions on the underlying data-generating process to estimate bounds for a causal query of interest \cite{Jesson.2021}. Knowing that the bounds are above or below zero is often sufficient for consequential decision-making \cite{Kallus.2019}. For example, knowledge about a positive treatment effect may be sufficient for a physician to prescribe a treatment. To the best of our knowledge, our work is unique in two ways. First, while there is rich literature on partial identification \citep[e.g.,][]{Manski.1990, Kilbertus.2020, Padh.2023}, no previous works allow for observational datasets from \emph{multiple environments}. Second, the \emph{derivation} of bounds for specific settings is a common theme in the partial identification literature \citep[e.g.,][]{Duarte.2023}, yet the effective estimation of such bounds is often not the focus. Here, a novelty of ours is that we provide flexible meta-learners that can be combined with arbitrary machine learning models for estimating such bounds in our setting. 

In this paper, we make three \textbf{contributions}:\footnote{Code is available at \url{https://github.com/JSchweisthal/BoundMetaLearners}.} (1)~We show that under similar assumptions as in related work \citep{Kallus.2018e}, the variable representing the environment can be interpreted as an instrument variable (IV). Then, we generalize previous results from the partial identification literature for binary IVs \cite{Balke.1997, swanson2018partial} to our setting to obtain bounds for CATE. (2)~We propose novel meta-learners to \emph{estimate} the bounds from observational data. Importantly, our meta-learners are model-agnostic and can be used in combination with any machine learning model. (3)~We provide theoretical results for our meta-learners by showing consistency and double robustness properties. Finally, we confirm the effectiveness of our meta-learners by performing various experiments using both simulated and real-world data.

We now give an \textbf{intuition} behind our bounds, and, thereby, we explain that access to a discrete variable representing the environment can indeed help with partial identification of the CATE, even under violations of overlap and unobserved confounding. Fig.~\ref{fig:intro} shows two different environments, namely, \textcolor{blue}{hospital $e$} and \textcolor{BrickRed}{hospital $j$}. Both have different treatment assignment mechanisms (i.e., \textcolor{blue}{$\pi_1^e$} and \textcolor{BrickRed}{$\pi_1^j$}, respectively), which are plotted against some patient characteristic $X$ (e.g., patient age). Let us highlight three regions: In $\mathcal{A}_1$ and $\mathcal{A}_3$, the treatment assignments are similar, because of which, unfortunately, we have wide bounds for the CATE. However, in $\mathcal{A}_2$, the overlap assumption is violated in both environments, but at \emph{opposite ends}: in \textcolor{blue}{hospital $e$}, middle-aged patients are almost always treated, while, in \textcolor{BrickRed}{hospital $j$}, they are rarely treated. Hence, we can overcome the limited overlap by \emph{combining} data across both environments. Furthermore, we can deduce claims regarding the ``strength'' of unobserved confounding in the region $\mathcal{A}_2$, as there can \emph{not} exist an unobserved variable with a strong influence on the treatment assignment because this would imply more observed variation in the prescribed treatments. Our intuition is that violations of overlap and unconfoundedness \emph{exclude each other}, so that we obtain tight bounds for the CATE.

\section{Related work}
\label{sec:rw}

Our work draws from multiple streams of related literature, which we list in the following. Additional literature is in Appendix~\ref{app:rw}.

\textbf{Treatment effect estimation across different environments:} Several works focus on treatment effect estimation from different environments. (i)~One stream focuses on randomized data, typically to combine randomized and observational datasets \citep[e.g.,][]{Kallus.2018e,Athey.2020,Ghassami.2022b,Hatt.2022, Imbens.2022}. %Yet, this requires access to randomized data, whereas we focus on only observational data.
(ii)~Another stream uses purely observational data from different environments. For example, previous works offer a theory for the transportability of causal effects across different environments \citep{Bareinboim.2016} or methods for transfer learning \cite{Bica.2022} and for detecting unobserved confounding \cite{Karlsson.2023}. However, none of these works estimates \emph{bounds} for the CATE across multiple environments \emph{under violations of assumptions}. 

\textbf{Meta-learners for CATE estimation:} Model-agnostic meta-learners, in particular two-stage learners, are commonly used in the standard setting for CATE estimation without unobserved confounding and achieve state-of-the-art performance both theoretically and empirically \citep{Foster.2019, Curth.2021}. The basic idea of two-stage learners is to estimate the CATE directly via an additional second-stage regression using a constructed pseudo-outcome. Several learners have been developed, which primarily differ in terms of the pseudo-outcome used: (i)~the regression-adjustment learner (called RA- or X-learner) \cite{Kunzel.2019, Curth.2021}; the inverse-propensity weighted learner (IPW-learner) \cite{Curth.2021}; the doubly robust learner (DR-learner) \citep{Kennedy.2023}; and the R-learner \cite{Nie.2021}. There also exist some efforts to generalize the standard meta-learners to different causal inference settings \citep[e.g., with instrumental variables,][]{Syrgkanis.2019, Frauen.2023b}. However, these meta-learners are all designed for estimating CATE under \emph{point} identification but \underline{not} \emph{partial} identification. 

We are only aware of one work that proposes a meta-learner for partial identification of CATE, namely, the B-learner \citep{Oprescu.2023}. However, the B-learner has been proposed for causal sensitivity analysis, a setting with a fundamentally different set of assumptions: the B-learner requires prior knowledge that limits the maximum strength of the confounding. In contrast, we do \emph{not} make assumptions that limit the strength of the confounding but instead assume that observational data is collected from multiple environments. Because of that, the B-learner is \underline{not} applicable in our work.  

\textbf{Partial identification of treatment effects:}
In the context of causal inference from a single observational dataset with unobserved confounding, several works have proposed procedures for partial identification of treatment effects. 

\citet{Manski.1990} was the first to obtain bounds for (conditional) average treatment effects under the assumption of bounded outcomes. \citet{Balke.1997} derived tighter bounds for average treatment effects in a similar setting than ours, but with binary variables. \citet{Duarte.2023} proposed a general procedure for deriving bounds in discrete structural causal models. In settings with continuous variables, recent works leverage machine learning approaches to learn bounds \cite{Kilbertus.2020, Hu.2021, Balazadeh.2022, Chen.2023, Padh.2023}. Note that these methods focus primarily on \emph{identification} results (i.e., deriving bounds), while we develop meta-learners for \emph{estimating} bounds. An exception is the estimator from \cite{levis2023covariate}, which however targets bounds for \emph{average} treatment effects.

Finally, a related but distinct stream of literature leverages machine learning for causal sensitivity analysis to obtain bounds under so-called sensitivity models, which make assumptions that limit the strength of unobserved confounding \cite{Jesson.2021, Dorn.2022, Dorn.2022b, Yin.2022, Frauen.2023c, Jin.2023, frauen2024neural}. Again, we do not make an assumption that limits the strength of unobserved confounding, because of which the tasks/bounds are not comparable. 
On top of that, to the best of our knowledge, none of these works considers effective \emph{estimation} of bounds for CATE via meta-learners.

\textbf{Research gap:} We focus on CATE estimation from observational data across multiple environments under violations of standard assumptions. To the best of our knowledge, we are the first to study partial identification in settings with \emph{multiple} environments. Further, we are the first to develop effective \emph{meta-learners} in our partial identification setting to estimate bounds. 

\section{Problem setup}
\label{sec:prob_setup}

\textbf{Setting:} We consider a setting with multiple environments (e.g., different hospitals, different physicians, or different countries), which, for simplicity, we denote by a discrete environment variable $E \in \{0, \dots, k\}$. We further have access to an observational dataset $\mathcal{D} = \left\{e_i, x_i, a_i, y_i\right\}_{i=1}^{n}$ of size $n$. The data are sampled i.i.d. from a population $(E, X, A, Y) \sim \mathbb{P}$, with patient covariates $X \in \mathcal{X} \subseteq \R^p$, discrete treatments $A \in \mathcal{A} \subseteq \N$, and bounded outcomes $Y \in \mathcal{Y} \subseteq [s_1, s_2] \subseteq \R$. The causal structure is shown in Fig.~\ref{fig:causal_graph}. In particular, we assume that $E$ is an instrumental variable that has an effect on the treatment $A$ but no direct effect on the outcome $Y$ that is not mediated via $A$. 
%\TODO{vermutlich nciht sagen?} For $k=1$, our setting is essentially the standard setting \TODO{\citep[e.g.,][]{}}.

Our setting is relevant for a variety of practical applications where observational data are collected from different environments. Consider electronic healthcare records from patients, where $X$ includes patient characteristics such as age and gender, $A$ is an indicator of whether a patient has been prescribed medical treatment, and $Y$ is a health outcome such as heart rate or blood pressure. Further, let $E$ correspond to an environment in which the electronic health records are collected such as the hospital. Such settings are common in medical research to understand the effectiveness of treatments \citep[e.g.,][]{Huang.2022}. 

\begin{figure}[ht]
  \centering
  \includegraphics[width=0.95\linewidth]{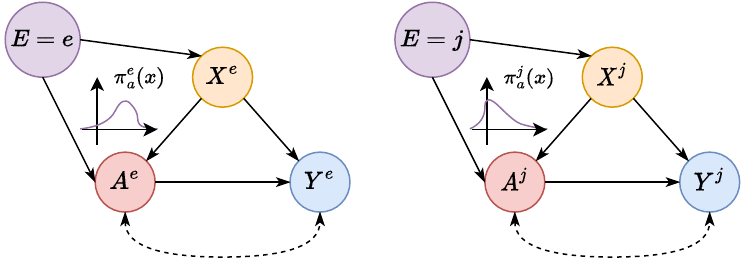}
  \caption{Causal graphs for different environments $e$ and $j$. We assume that the causal structure between $e$ and $j$ remains unchanged but we allow for different treatment assignment mechanisms (propensity scores) $\pi^e_a(x)$ and $\pi^j_a(x)$.  The dotted arrows indicate potential unobserved confounding.}
  \label{fig:causal_graph}
% \vspace{-0.3cm}
\end{figure}

\textbf{Notation:} We define the environment-specific \emph{propensity score} as $\pi^e_a(x) = \mathbb{P}(A = a \mid X = x, E = e)$, which denotes the treatment assignment mechanism that is observed in environment $e$. Furthermore, we define the environment-specific \emph{response surfaces} via $\mu^e_a(x) = \E[Y \mid X = x, A = a, E = e]$. Importantly, we assume that both the propensity scores and the response surfaces may differ across environments, i.e., $\pi^e_a(x) \neq \pi^j_a(x)$ and $\mu^e_a(x) \neq \mu^j_a(x)$ for environments $e \neq j$. For example, the standard of care might vary across countries, leading to different treatment policies and thus propensity scores. Finally, we define the environment probability as $\delta_e(x) = \prob(E = e \mid X = x)$.

\textbf{Assumptions:} We use the potential outcomes framework to formalize our causal inference problem \cite{Rubin.1974}. Let $Y^e(a) \in \mathcal{Y}$ denote the potential outcome in environment $E=e$ for a treatment intervention $A = a$. In this paper, we impose the following assumptions on the data-generating process across environments.
\begin{assumption}[Consistency across environments]\label{ass:consistency} For each environment $e$, observing $A = a$ implies $Y^e(a) = Y$.
\end{assumption}
\begin{assumption}[Environment-agnostic oracle response surfaces]\label{ass:main}
We assume that $\E[Y^e(a) \mid X = x, E = e] = \E[Y^j(a) \mid X= x, E = j]$ holds for all environments $0 \leq e, j \leq k$, treatments $a \in \mathcal{A}$, and covariates $x \in \mathcal{X}$.
\end{assumption}
\begin{assumption}[Common support]\label{ass:support}
We assume that $\delta_e(x) > 0$ holds for all environments $0 \leq e \leq k$ and covariates $x \in \mathcal{X}$.
\end{assumption}
Assumption~\ref{ass:consistency} (consistency) excludes spillover effects across individuals and is a standard assumption in the causal inference literature \cite{Shalit.2017, Wager.2018}. 
Assumption~\ref{ass:main} implies that a treatment results in the same expected potential outcome across environments.\footnote{Prior literature often defines instrumental variables $E$ via $Y(a) \indep E \mid X$ \cite{Pearl.2009}. Note that this assumption is stronger than our Assumption~\ref{ass:main}. Hence, our results can be applied to any setting where $E$ is a discrete instrument in the classical sense, e.g., RCTs with non-compliance.} In particular, the outcome-generating causal mechanisms coincide in expectation for each environment. 
Similar assumptions are commonly made in the literature to ensure the transferability of treatment effect estimates. Finally, Assumption~\ref{ass:support} ensures that the populations of all environments are supported on a common domain, i.e., there are no types of patients that only exist in certain environments. Thus, we are consistent with prior literature on causal inference using combined datasets \cite{Kallus.2018e, Shi.2021, Hatt.2022}. 

\textbf{Target estimand:}
Using Assumption~\ref{ass:main}, we can define the oracle response surface as $\widetilde{\mu}_{a}(x) = \E[Y^e(a) \mid X = x, E = e]$, which is independent of the environment $e$. We are interested in the (environment-agnostic) conditional average treatment effect (CATE) 
\vspace{-0.1cm}
\begin{equation}\label{eq:cate_def}
\tau_{a_1, a_2}(x) = \widetilde{\mu}_{a_1}(x) - \widetilde{\mu}_{a_2}(x),   
\end{equation}
which quantifies the difference in expected potential outcomes for two treatments $a_1, a_2 \in \mathcal{A}$. The CATE is of interest in various applications, e.g., personalized medicine \cite{feuerriegel2024causal}, because it captures treatment effect heterogeneity by conditioning on the covariates $x$.

\textbf{Violations of standard causal inference assumptions:} 
The standard causal inference literature imposes the following two additional assumptions to identify the CATE $\tau_{a_1, a_2}(x)$ from the observational data distribution \cite{Wager.2018, Curth.2021}: 
\vspace{-0.3cm}
\begin{enumerate}[label=(\roman*)]
    \item \emph{Overlap within envionments}: $\pi^e_a(x) > 0$; and
\vspace{-0.3cm}
    \item \emph{Unconfoundedness}: $Y^e(a) \indep A \mid X = x$
\end{enumerate}
\vspace{-0.3cm}
for all $x \in \mathcal{X}$ and $a \in \mathcal{A}$.

Overlap implies that each individual with covariates $x$ must have a non-zero probability of receiving every treatment $a$. Unconfoundedness implies that the observed covariates $X$ capture all confounding between the treatment $A$ and outcome $Y$. Under overlap and unconfoundedness, the oracle response functions are identified via 
$\widetilde{\mu}_{a}(x) = \mu^e_a(x)$ and, hence, coincide with their observed counterparts for all environments $e$ \cite{Shalit.2017}. In particular, the CATE is identified via $\tau_{a_1, a_2}(x) = \mu^e_{a_1}(x) - \mu^e_{a_2}(x)$.

However, both overlap and unconfoundedness are violated in many practical settings. (i)~The overlap violation is almost always violated in settings with high-dimensional covariates $X$, which is common in modern settings with available text/image data, or other electronic healthcare records \cite{DAmour.2021, schweisthal2024reliable}. (ii)~The unconfoundedness assumption is highly unrealistic in medical settings, where confounders such as socio-economic status are rarely recorded in electronic health records \cite{Adler.2015}. If either overlap or unconfoundedness (or both) are violated, neither response functions $\widetilde{\mu}_{a}(x) \neq \mu^e_a(x)$ nor the CATE are identified. Then, unbiased estimation is impossible \cite{Pearl.2009}. 

In the following, we relax assumptions related to overlap and unconfoundedness and focus on \emph{partial identification}.

\section{Partial identification of CATE}
\label{sec:method}

We now leverage the idea that the environment variable $E$ acts as an instrumental variable (IV) and generalize results from the literature on partial identification for IVs \citep{Balke.1997, swanson2018partial} that hold for binary variables and average treatment effects. To do so, we first recall a known result from the literature on bounding treatment effects without multiple environments.
\begin{lemma}\label{lem:manski}\cite{Manski.1990}
For any environment $e$, the oracle response surfaces are bounded under Assumption~\ref{ass:consistency} via
\begin{equation}
    \widetilde{\mu}_{a}(x) \leq \pi^e_a(x) \, \mu^e_a(x) + (1 - \pi^e_a(x)) \, s_2
\end{equation}
and
\begin{equation}
    \widetilde{\mu}_{a}(x) \geq \pi^e_a(x) \, \mu^e_a(x) + (1 - \pi^e_a(x)) \, s_1,
\end{equation}
where $[s_1, s_2]$ denotes the support of $Y$.
\end{lemma}
\begin{proof}
See Appendix~\ref{app:proofs}.
\end{proof}
The intuition behind Lemma~\ref{lem:manski} is as follows: Whenever we have overlap violations and the propensity score $\pi^e_a(x)$ is large for a treatment $a$, most of the randomness in the treatment is removed by conditioning on the observed covariates $x$. Hence, there can \emph{not} be unobserved confounders that have a strong influence on the treatment assignment. This is reflected in the bounds from Lemma~\ref{lem:manski}, which become tighter whenever $\pi^e_a(x)$ becomes larger.

Previously, \citet{Manski.1990} proposed to obtain an upper bound for the CATE $\tau_{a_1, a_2}(x) = \widetilde{\mu}_{a_1}(x) - \widetilde{\mu}_{a_2}(x)$ by combining the upper and lower bound from Lemma~\ref{lem:manski} for the different treatments $a_1$ and $a_2$. This results in
\begin{multline}
\label{eq:manski}
\tau_{a_1, a_2}(x) \leq \pi^e_{a_1}(x) \, \mu^e_{a_1}(x) + (1 - \pi^e_{a_1}(x)) \, s_2 \\
\quad - \pi^e_{a_2}(x) \, \mu^e_{a_2}(x) - (1 - \pi^e_{a_2}(x)) \, s_1   
\end{multline}
for the upper bound. An analogous result can be obtained for the lower bound by swapping the support points $s_1$ and $s_2$.

One drawback of the Manski bound in Eq.~\eqref{eq:manski} is that it is not particularly tight. When subtracting the lower bound from the upper bound, we obtain a constant tightness of $s_2 - s_1$, which does \emph{not} depend on the covariates $x$ and \emph{not} on the propensity score $\pi^e_a(x)$. In particular, the tightness equals the full support of the outcome variable $Y$. 

\textbf{Extending the Manski bounds:} The above motivates us to leverage the different propensity scores across environments to obtain tighter bounds (see Fig.~\ref{fig:intro}). By combining the bounds from Lemma~\ref{lem:manski} for the oracle response functions across environments, we derive the following result.

\begin{theorem}\label{thrm:bounds}
Under Assumptions~\ref{ass:consistency} and \ref{ass:main}, the CATE is bounded via
\begin{equation}
b^-(x) \leq \tau_{a_1, a_2}(x) \leq b^+(x),
\end{equation}
where
\begin{equation}\label{eq:bounds}
    b^+(x) = \min_{e, j} b^+_{e,j}(x) \quad \text{and} \quad b^-(x) = \max_{e, j} b^-_{e,j}(x) 
\end{equation}
with
\begin{multline}
    b^+_{e,j}(x) =\pi^e_{a_1}(x) \mu^e_{a_1}(x) + (1 - \pi^e_{a_1}(x)) s_2  \\ 
    \quad - \pi^j_{a_2}(x) \mu^j_{a_2}(x) - (1 - \pi^j_{a_2}(x)) s_1 ,
\end{multline}
%and
\vspace{-0.4cm}
\begin{multline}
    b^-_{e,j}(x) = \pi^e_{a_1}(x) \mu^e_{a_1}(x) + (1 - \pi^e_{a_1}(x)) s_1 
    \\  
    \quad - \pi^j_{a_2}(x) \mu^j_{a_2}(x) - (1 - \pi^j_{a_2}(x)) s_2.
\end{multline}
\end{theorem}
\begin{proof}
See Appendix~\ref{app:proofs}.
\end{proof}
\vspace{-0.3cm}

The bounds from Theorem~\ref{thrm:bounds} are a generalization of the so-called \emph{natural bounds} from the literature on partial identification with IVs and discrete variables \cite{swanson2018partial}. Unlike other bounds such as the ones from \citet{Balke.1997}, these can be straightforwardly extended to continuous outcomes and thus our multiple environments setting via Theorem~\ref{thrm:bounds}. We refer to \citet{swanson2018partial} for a detailed overview on bounds in instrumental settings (with discrete variables). Finally, another advantage of the natural bounds is that they are in closed-form and thus allow for estimation via meta-learners (see Sec.~\ref{sec:meta_learners}).

\textbf{Tightness of the bounds:} Note that our bounds from Theorem~\ref{thrm:bounds} satisfy
\begin{multline}\label{eq:bound_width}
b^+(x) - b^-(x) \\ 
\leq \min_{e, j}\left\{(s_2 - s_1) (2 - \pi^e_{a_1}(x) - \pi^j_{a_2}(x))\right\},
\end{multline}
which shows that they improve on the Manski bounds from Eq.~\eqref{eq:manski} in terms of tightness whenever $\pi^e_{a_1}(x) + \pi^j_{a_2}(x) > 1$. 

\textbf{Remark:} Our bounds become tight whenever there exist two environments $e$ and $j$ so that both $\pi^e_{a_1}(x)$ and $\pi^j_{a_2}$ are large. In other words, we have a rather surprising implication: \emph{violations of overlap can be beneficial to increase the tightness} of our bounds as long as the respective treatment assignments across environments change sufficiently. 

Intuitively, the degree to which the propensity scores across environments differ can also be viewed as the strength of dependence between the environment variable $E$ and the treatment $A$. In the literature on IV regression, it is well-known that so-called weak instruments $E$ that are only weakly correlated with $A$ may lead to biased estimates \citep{Angrist.2008}. Our bounds from Theorem~\ref{thrm:bounds} automatically account for this issue and become wider whenever $E$ is a weak instrument and the propensity scores $\pi^e_{a_1}(x)$ and $\pi^j_{a_2}$ across environments become similar.

\section{Meta-learners for estimating the bounds}\label{sec:meta_learners}

We now develop meta-learners for estimating our bounds from Theorem~\ref{thrm:bounds}, which denote quantities in population. In the following, we thus propose model-agnostic meta-learners that can be used in combination with any machine-learning method. In this section, we describe the estimation of our bounds for binary treatments $A \in \{0, 1\}$, so that $a_1 = 1$ and $a_2 = 0$. We also provide a theoretical analysis (Sec.~\ref{sec:theoretical_results}) and an implementation using neural networks (Sec.~\ref{sec:neural}).

For simplicity, we will focus on obtaining estimators $\hat{b}^+(x)$ for the upper bound $b^+(x)$. Estimators for the lower bound $b^-(x)$ can be obtained the same way by interchanging $s_1$ and $s_2$ (see Theorem~\ref{thrm:bounds}).

\subsection{Na{\"i}ve plug-in learner}

A straightforward way to obtain an estimator for $b^+(x)$ is the so-called na{\"i}ve \emph{plug-in} approach. In this, we first obtain estimators $\hat{\mu}^e_{a}(x)$ and $\hat{\pi}_a^e(x)$  of the \emph{nuisance functions} $\mu^e_{a}(x)$ and $\pi_a^e(x)$ for all $a, e \in \{0,1\}$. Note that estimating $\mu^e_{a}(x)$ is a regression task and estimating $\pi_a^e(x)$ is a classification task, which means that off-the-shelf machine-learning algorithms can be applied. Then, we directly plug the estimated nuisance functions into Eq.~\eqref{eq:bounds}, which yields
%\vspace{-0.3cm}
\begin{multline}\label{eq:plugin_learner}
    \hat{b}^+(x) =\min_{e,j} \Big\{\hat{\pi}_1^e(x) \hat{\mu}^e_{1}(x) + \hat{\pi}_0^e(x)  \, s_2  \\  
    - \hat{\pi}_0^j(x) \, \hat{\mu}^j_{0}(x) - \hat{\pi}_1^j(x) s_1\Big\}.
\end{multline}
%\vspace{-0.3cm}
We call the estimator from Eq.~\eqref{eq:plugin_learner} the \emph{na{\"i}ve plugin-learner}. 

The general approach behind plug-in learners can suffer from a so-called plug-in bias \citep{Kennedy.2022} that can limit estimation performance. As a remedy, we develop novel two-stage learners for our task in the following.

\subsection{Two-stage learners}\label{sec:two-stage-learners}

We now aim to address the drawbacks of plug-in learners by proposing so-called \emph{two-stage learners} that directly estimate the bounds from Theorem~\ref{thrm:bounds}. While the plug-in learner plugs the estimated nuisance function into Eq.~\eqref{eq:plugin_learner} to estimate $b^+_{e,j}(x)$, we now aim to estimate $b^+_{e,j}(x)$ directly via a second-stage learner for all environment combinations $e, j \in \{0, 1\}$.  If the bounds $b^+_{e,j}(x)$ are easier to estimate directly than the corresponding nuisance functions (e.g., due to cancellation effects), such an approach should perform better than the plug-in learner. We do this by constructing pseudo-outcomes that are equal to the bound we want to estimate in expectation, which then can be used as a second-stage regression objective. As a result, we obtain novel two-stage learners for estimating $b^+_{e,j}(x)$ and thus $b^+(x)$ (instead of the non-identifiable CATE $\tau_{1, 0}(x)$). 

\textbf{Types of learners:} We distinguish between two different cases: (1)~within-environment bounds $b^+_{e,e}(x)$, which combine the response function bounds from Lemma~\ref{lem:manski} within the same environment $e$, and (2)~cross-environment bounds $b^+_{e,j}(x)$, which combine the response function bounds across different environments $e \neq j$. Consequently, we yield two different approaches for the second-stage learners, which we call (1)~\textbf{WB-learner} and (2)~\textbf{CB-learner} (where the latter comes in different variants called CB-PI, CB-RA-, CB-IPW-, and CB-DR-learner). Once all second-stage learners $\hat{b}^+_{e,j}(x)$ are fitted, the within-environment bounds and the cross-environment bounds are combined via $\hat{b}^+(x) = \min_{e, j} \hat{b}^+_{e,j}(x)$ to obtain the final bound estimator for CATE. The full procedure is shown in Algorithm~\ref{alg:meta_algorithm}. 

\subsubsection{WB-learner} Whenever we consider a single environment $e=j$, we define the pseudo-outcome
\begin{equation}\label{eq:wb_learner}
    \hat{B}^{+\mathrm{WB}}_{e} = \mathbbm{1}\{E=e\} \left(A Y + (1-A)s_2 - (1 - A)Y - A s_1 \right).
\end{equation}
Then, we use the pseudo-outcome $\hat{B}^{+\mathrm{WB}}_{e}$ to estimate $\hat{b}^+_{e,e}(x) = \hat{\E}[\hat{B}^+_{e} \mid X = x, E = e]$ via a second-stage learner $\hat{\E}$ conditional on $E = e$. That is, we only use the data from environment $e$ for the estimator $\hat{b}^+_{e,e}(x)$. We call this the \emph{within-environment bound-learner} (WB-learner).

\subsubsection{CB-learners}\label{sec:cb-learners}

Whenever $e \neq j$, we need to combine data from different environments $e$ and $j$ to estimate the cross-environment bounds $b^+_{e,j}(x)$. For this purpose, we show that the estimation of $b^+_{e,j}(x)$ can be cast into a standard CATE estimation problem with a transformed outcome and treatment variable $\widetilde{Y}$ and using the environment $E$ as a (multi-valued) treatment. That is, we define
\begin{multline}\label{eq:cb-pseudo-outcome}
    \widetilde{Y}^+_{e,j} = \mathbbm{1}\{E=e\} \left(AY + (1-A)s_2 \right) \\ + \mathbbm{1}\{E=j\} \left((1-A)Y + A s_1\right) ,
\end{multline}
and we denote the \emph{transformed response functions} as
${r}^+_{\ell}(x) = {\E}[\widetilde{Y}^+_{e,j} \mid X = x, E = \ell]$ for $\ell \in \{e, j\}$.

\textbf{High-level approach:} We proceed in two stages. In \textbf{stage 1}, we obtain estimators $\hat{r}^+_{\ell}$ and $\hat{\delta}_e(x)$ of the nuisance functions ${r}^+_{\ell}$ and $\delta_\ell(x)$. In \textbf{stage 2}, we treat the estimation of $b^+_{e,j}(x)$ as a standard CATE estimation task using $\hat{r}^+_{\ell}$ as estimated response functions and $\hat{\delta}_\ell(x)$ as the estimated propensity score. Then, we adopt the classical meta-learners (\mbox{plugin-,} \mbox{RA-,} \mbox{IPW-,} and DR-learner \cite{Curth.2020, Kennedy.2023b}) for multi-valued treatments \cite{acharki2023comparison} (i.e., our environments). For guarantees on this approach, we refer to Sec.~\ref{sec:theoretical_results}.

(i)~\textbf{CB-PI-learner}: The \emph{cross-environment-bound plugin-learner} is defined via
\begin{equation}
\hat{b}^{+, \mathrm{PI}}_{e,j}(x) = \hat{r}^+_{e}(x) - \hat{r}^+_{j}(x),
\end{equation}
which is analogous to the plugin-learner for standard CATE estimation \cite{Curth.2020}.

(ii)~\textbf{CB-RA-learner}: The \emph{cross-environment-bound regression-adjustment learner} is defined via the pseudo outcome
{\tiny
\begin{multline}
    \hat{B}^{+\mathrm{RA}}_{e, j} = \mathbbm{1}\{E=e\}\left(\widetilde{Y}^+_{e,j} - \hat{r}^+_{j}(x) \right)  
    + \mathbbm{1}\{E=j\} \left(\hat{r}^+_{e}(x) -\widetilde{Y}^+_{e,j}\right) \\ + \mathbbm{1}\{E\neq e\} \mathbbm{1}\{E\neq j\} \left(\hat{r}^+_{e}(x) - \hat{r}^+_{j}(x) \right),
\end{multline}}%
and the second-stage pseudo-outcome regression $\hat{b}^+_{e,j}(x) = \hat{\E}[\hat{B}^{+\mathrm{RA}}_{e, j} \mid X = x]$.

(iii)~\textbf{CB-IPW-learner}: The \emph{cross-environment-bound inverse propensity weighted learner} is defined via
{\scriptsize
\begin{equation}
    \hat{B}^{+\mathrm{IPW}}_{e, j} = \left(\frac{\mathbbm{1}\{E=e\}}{\hat{\delta}_e(x)}
     - \frac{\mathbbm{1}\{E=j\}}{\hat{\delta}_j(x)}\right) \widetilde{Y}^+_{e,j},
\end{equation}}%
and the second-stage pseudo-outcome regression $\hat{b}^+_{e,j}(x) = \hat{\E}[\hat{B}^{+\mathrm{IPW}}_{e, j} \mid X = x]$.

(iv)~\textbf{CB-DR-learner}: The \emph{cross-environment-bound doubly robust learner} is defined via
{\scriptsize
\begin{equation}
    \hat{B}^{+\mathrm{DR}}_{e, j} = \hat{B}^{+\mathrm{IPW}}_{e, j} + \left(1 - \frac{\mathbbm{1}\{E=e\}}{\hat{\delta}_e(x)}\right) \hat{r}^+_{e}(x) - \left(1 - \frac{\mathbbm{1}\{E=j\}}{\hat{\delta}_j(x)}\right) \hat{r}^+_{j}(x)
\end{equation}}%
and the second-stage pseudo-outcome regression $\hat{b}^+_{e,j}(x) = \hat{\E}[\hat{B}^{+\mathrm{DR}}_{e, j} \mid X = x]$.

(v)~\textbf{Further meta-learners}: We emphasize that, in principle, \emph{any} existing meta-learner for CATE estimation can be applied for estimating the cross-environment bounds. Further learners not specifically discussed here include the U-learner \cite{Kunzel.2019} and the R-learner \cite{Nie.2021} (in the case of a binary environmental variable $E$).

\vspace{-0.3cm}

\begin{algorithm}\label{alg:meta_algorithm}
\DontPrintSemicolon
\scriptsize
\caption{Two-stage learners for estimating bounds}
\label{alg:learning_full}
\SetKwInOut{Input}{Input}
\SetKwInOut{Output}{Output}
\Input{~\mbox{observational data $(E, X, A, Y)$, method $m \in \{\mathrm{RA}, \mathrm{IPW}, \mathrm{DR}\}$}}
\Output{estimated upper bound $\hat{b}^+(x)$ (for $\hat{b}^-(x)$, interchange $s_1$ and $s_2$)}
\tcp{Stage 1 (nuisance estimation)}
$\hat{r}^+_{\ell}(x) \gets \hat{\E}[\widetilde{Y}^+_{e,j} \mid X = x, E = \ell]$\;
$\hat{\delta}_\ell(x) \gets \hat{\prob}(E = \ell \mid X = x)$\;
\tcp{Stage 2}
 \For{$e, j \in \{0, \dots, k\}$}{
 \uIf{$e = j$}{
    $\hat{b}^+_{e,e}(x) = \hat{\E}[\hat{B}^{+\mathrm{WB}}_{e} \mid X = x, E = e]$\;
  }
  \Else{
    $\hat{b}^+_{e,j}(x) = \hat{\E}[\hat{B}^{+m}_{e, j} \mid X = x]$\;
  }
}
\tcp{Final bound}
$\hat{b}^+(x) = \min_{e, j} \hat{b}^+_{e,j}(x)$\;
\end{algorithm}

\subsection{Theoretical guarantees}
\label{sec:theoretical_results}

The following result analyzes the consistency of our proposed second-stage learners depending on whether the nuisance parameters are correctly specified.

\begin{theorem}[Consistency and double robustness]\label{thrm:consistency}
The meta-learners are consistent estimators of the within and cross-environment bounds. For the WB-learner, it holds that
\begin{equation}
    b^+_{e,e}(x) = \E[\hat{B}^{+\mathrm{WB}}_{e, e} \mid X = x, E = e]
\end{equation}
for $e \in \{0, 1\}$. For the CB-$m$-learner with $m \in \{\mathrm{RA}, \mathrm{IPW}, \mathrm{DR}\}$, we obtain
\begin{equation}
b^+_{e,j}(x) = \E[\hat{B}^{+m}_{e, j} \mid X = x]
\end{equation}
for $e, j \in \{0, 1\}$ whenever, for all $e, a \in \{0, 1\}$, one of the following cases hold: (i)~$\hat{r}^+_{\ell}(x) = {r}^+_{\ell}(x)$ for $m = \mathrm{RA}$; (ii)~$\hat{\delta}_\ell(x) = {\delta}_\ell(x)$ for $m = \mathrm{IPW}$; or (iii)~\underline{either} $\hat{r}^+_{\ell}(x) = {r}^+_{\ell}(x)$ \underline{or} $\hat{\delta}_\ell(x) = {\delta}_\ell(x)$ for $m = \mathrm{DR}$. 
\end{theorem}
\begin{proof}
    See Appendix~\ref{app:proofs}.
\end{proof}

\vspace{-0.1cm}

Note that the CB-DR-learner satisfies a double robustness property analogously to the standard DR-learner for CATE estimation \citep{Kennedy.2023}. In particular, it allows for a misspecification of either $\hat{r}^+_{\ell}(x)$ or $\hat{\delta}^\ell(x)$ as long as the other nuisance function is correctly specified. Furthermore, the WB-learner is always consistent because the corresponding pseudo-outcome from Eq.~\eqref{eq:wb_learner} does not depend on any nuisance estimators but just on the observational data. This is different for cross-environment bounds $b^+_{e,j}(x)$, which would require observing the same data in two different environments. In that sense, the problem of estimating the $b^+_{e,j}(x)$ gives rise to the \emph{fundamental problem of causal inference}: potential outcomes corresponding to different environments are never observed simultaneously.

\subsection{Implementation using neural networks}
\label{sec:neural}

In general, our proposed meta-learners are model-agnostic in both stages and thus work with arbitrary machine learning models. Here, we provide a flexible implementation using neural networks. Specifically, we adapt the implementations used in previous works for evaluating meta-learners for point-identified CATE estimation \cite{Curth.2021, Curth.2021c}. Importantly, we use similar network architectures and settings for all meta-learners. Thus, performance differences are decoupled from model choice, so that the source of gain is given by the meta-learners. We report the implementation details in Appendix~\ref{app:implementation}. Here, we want to emphasize that our experiments serve as a ``proof-of-demonstration'' to show the behavior of the different meta-learners in different settings, while the neural network architectures are not optimized for the respective tasks. Instead, in practice, the exact modeling choices can be flexibly adjusted to fit data properties and desired inductive biases. 

\section{Experiments}
\label{sec:experiments}

\begin{table}[]
    \centering
    \begin{tabular}{lcc}
\toprule
method     & Synthetic Data 1 & Synthetic Data 2 \\
\midrule
WB na{\"i}ve & 0.073 ± 0.031   & 0.075 ± 0.045 \\
WB           & 0.142 ± 0.069   & 0.130 ± 0.077 \\
\midrule
CB na{\"i}ve & 0.148 ± 0.098   & 0.156 ± 0.105 \\
CB-PI        & 0.125 ± 0.059   & 0.127 ± 0.063 \\
CB-RA        & 0.179 ± 0.089   & 0.119 ± 0.037 \\
CB-IPW  & \textbf{0.117 ± 0.057} & 0.165 ± 0.072 \\
CB-DR        & 0.132 ± 0.061   & \textbf{0.111 ± 0.069} \\
\bottomrule
\end{tabular}
    \caption{Mean and standard deviation of the RMSE over 5 random runs for synthetic datasets 1\&2.}
    \label{tab:synth}
\end{table}

We perform experiments on synthetic and real-world data to demonstrate that our meta-learners can effectively learn bounds in our setting with multiple environments. Synthetic data are commonly used to evaluate causal inference methods as it ensures that the causal ground-truth is available \cite{Curth.2021, Xu.2021} and thus allows us to examine the performance. Since, as in CATE estimation \cite{Curth.2021}, the performance of our different individual meta-learners depends strongly on the exact data-generating process, we refrain from heavy benchmarking. Instead, we show that our meta-learners reliably learn valid bounds and give short intuition to their behavior. 

\begin{figure*}[ht]
  \centering
\begin{subfigure}
  \centering
  \includegraphics[height=0.24\linewidth]{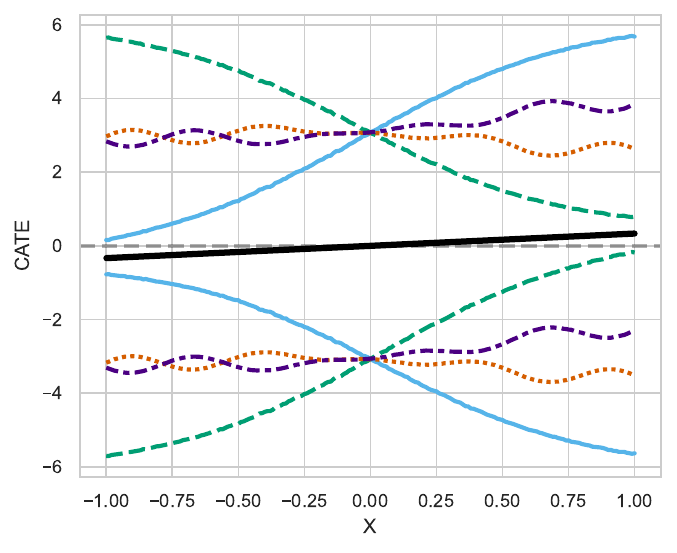}
\end{subfigure}%
\begin{subfigure}
  \centering
  \includegraphics[height=0.24\linewidth]{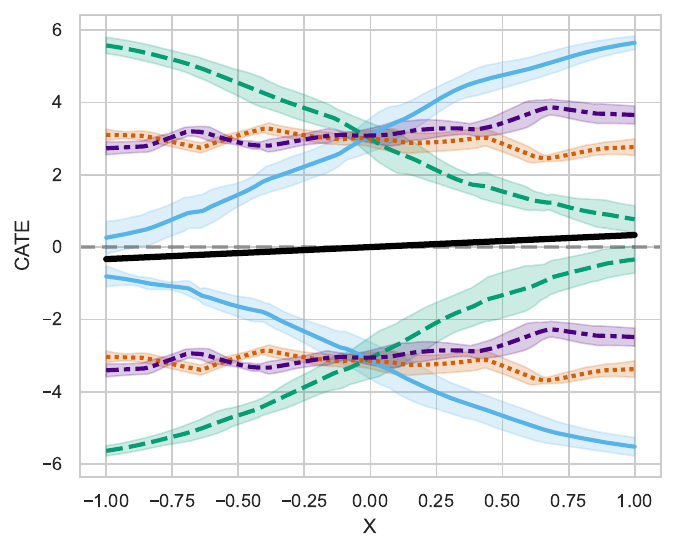}
\end{subfigure}
\begin{subfigure}
  \centering
  \includegraphics[height=0.24\linewidth]{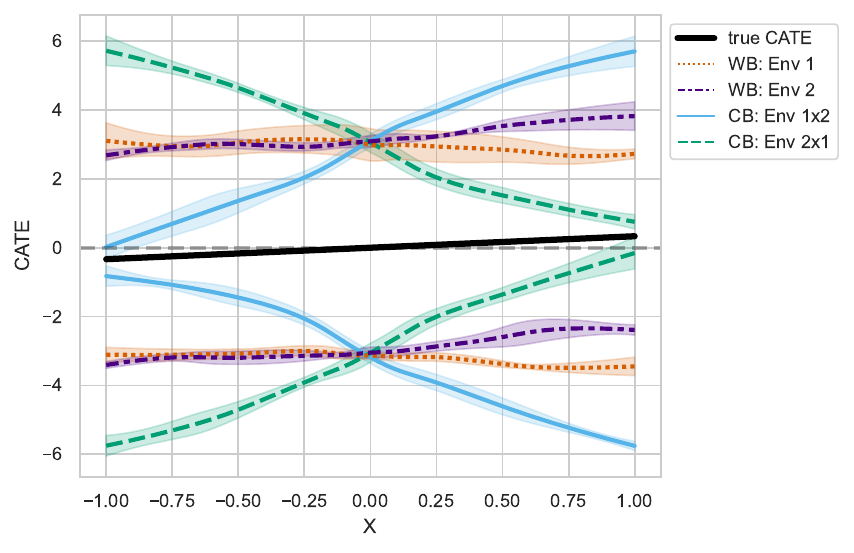}
\end{subfigure}
\vspace{-0.4cm}
\caption{Comparison of estimation methods for bounds based on synthetic dataset 2 (predicted bounds $mean \pm 3std$ over 5 runs). \textbf{Left:} Oracle bounds for within- (WB) and cross- (CB) environments. \textbf{Center:} Estimated bounds by the na{\"i}ve plug-in learner. \textbf{Right:} Estimated bounds by our two-stage meta-learners (here: WB-learner with CB-DR-learner).}
\label{fig:results_synthetic}
\vspace{-0.2cm}
\end{figure*}

\textbf{Synthetic data:} For synthetic data, unlike for real-world data, we have access to the ground-truth data-generating process. Hence, we can compare our meta-learners against the oracle CATE and the oracle bounds from Theorem~\ref{thrm:bounds} calculated with the ground-truth nuisance estimators. Here, we consider two settings where we vary the complexity of the environment probability $\delta_e(x)$ from rather simple (dataset 1) to more complex (dataset 2). Details on the data-generating mechanisms are in Appendix~\ref{app:sim}. In both cases, the treatment assignment tends to become more different across environments at the border of $X$, so that the cross-environment bounds should become tighter and more informative in these areas.

We report the results in terms of the root mean squared error (RMSE) to the oracle bounds in Table~\ref{tab:synth}. We further display illustrative insights into the true and predicted bounds of different meta-learners for dataset 2 in Figure~\ref{fig:results_synthetic}. We can observe the following: (i)~All of our meta-learners learn valid bounds reliably, as shown by low average and variation of the RMSE and comparable performances between the meta-learners. (ii)~Depending on the setting, different meta-learners may perform better or worse. As expected, in dataset 1 with simple $\delta_e(x)$, the CB-IPW learner performs best, while, in the more complex dataset 2, the CB-DR-learner shows the best performance. This is in line with previous work around meta-learners for CATE estimation \cite{Curth.2021}, while we show that the same holds true for meta-learners aimed at partial identification. (iii)~The cross-environment bounds can be especially helpful for learning tight and informative bounds, as demonstrated in the border regions of $X$ in Figure~\ref{fig:results_synthetic}. Hence, reliable estimation is particularly important in these areas, and, in the considered settings, our CB-learners help to improve performance over the na{\"i}ve plug-in learner.

Furthermore, in Figure~\ref{fig:results_synthetic}, we can give an intuition into the differences in bounds estimation between the na{\"i}ve plug-in learner and the two-stage learners. For the latter, we use the WB-learner in combination with the CB-DR-learner as a representative example for illustration. While still performing well, we can observe that the na{\"i}ve plugin-learner yields slightly less stable estimates with higher stand deviations, which is expected due to estimation errors in the nuisance estimation. The two-stage learners, in contrast, are less prone to minor errors in the nuisance estimators and, as expected, yield slightly smoother and more robust estimates of the bounds, as shown by lower stand deviations. This is also reflected by lower RMSE and lower variation in Table~\ref{tab:synth}. However, note that this can result in oversmoothing when the ground-truth bounds have a rather complex form, e.g., as for the within-environment bounds. Interestingly, as stated above, also here a similar behavior was observed for meta-learners for point identification. In sum, our results show that, depending on the data properties, different learners can be preferable for estimating bounds, and we provide a selection of different options with our novel meta-learners.

\textbf{Real-world data:} We now demonstrate the applicability of our meta-learners to real-world data. Since the ground-truth CATE is unobservable for real-world data, we refrain from benchmarking, but, instead, our primary aim is to demonstrate the practical value of our meta-learners. Here, we perform a case study using a dataset with COVID-19 hospitalizations in Brazil across different regions \cite{baqui2020ethnic}. We are interested in predicting the effect of comorbidity on the mortality of COVID-19 patients. For the environments, we use the regions of the hospitals in Brazil, which are split into North and Central-South. As observed confounders, we include age, sex, and ethnicity.

We report the estimates of the CATE bounds wrt. to age and averaged over the other confounders ($mean\pm std$), estimated by (i)~the na{\"i}ve plug-in learner and (ii)~the CB-PI-learner as another example for the two-stage learners. The results are displayed in Fig.~\ref{fig:results_covid}. As expected, for both learners, the best lower bounds are closer to zero than the best upper bounds for all ages, indicating that comorbidity has no (large) negative effect on mortality (i.e., no large positive effect on survival probability). However, given our data, we cannot prove that comorbidity has an effect on mortality at all. Further, we observe that the cross-environment bounds help to tighten the bounds. In comparison, our two-stage learners yield similar predictions to the na{\"i}ve plug-in learner, indicating robust estimation. In sum, the results demonstrate the applicability of our meta-learners to real-world data.

\begin{figure}[htb]
  \centering
\begin{subfigure}
  \centering
  \includegraphics[width=0.49\linewidth]{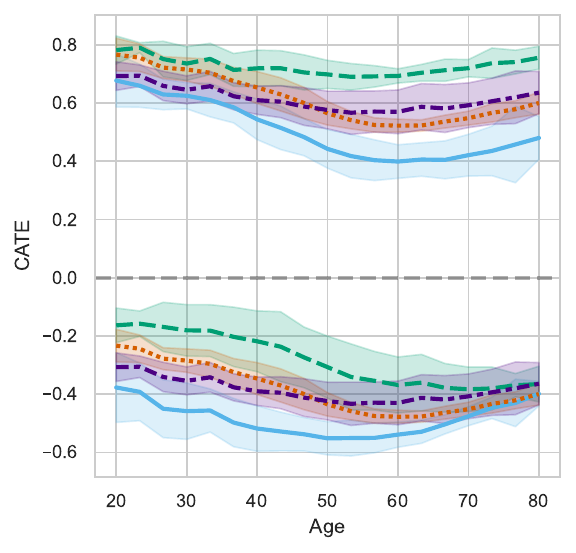}
\end{subfigure}%
\begin{subfigure}
  \centering
  \includegraphics[width=0.49\linewidth]{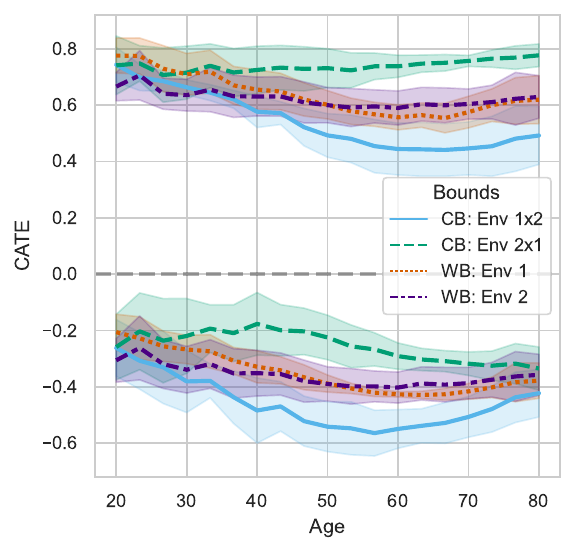}
\end{subfigure}
\vspace{-0.4cm}
\caption{Insights from real-world data: Effect of comorbidity on mortality in COVID-19 patients in Brazil. \textbf{Left:} Estimated bounds by na{\"i}ve plug-in learner. \textbf{Right:} Estimated bounds by our two-stage meta-learners (here: WB-learner and CB-PI-learner).}
\label{fig:results_covid}
\vspace{-0.2cm}
\end{figure}

\section{Applicability to IV settings}

Our meta-learners apply to any setting in which a discrete instrumental variable (IV), $E$, is available, which does not necessarily need to correspond to an environment indicator. A prominent example are {randomized controlled trials (RCTs) with non-compliance}. Note that RCTs with non-compliance do not need multiple environments but can have just a single environment. %As an example, consider a medical setting, where $X$ includes patient characteristics and $Y$ is a health outcome such as heart rate or blood pressure.
Then, $E$ corresponds to a randomized treatment prescription (e.g., by a physician) and $A$ indicates whether the patient complied and followed the treatment decision. Here, the randomized treatment assignment automatically renders $E$ independent of potential unobserved confounders, thus fulfilling Assumption~\ref{ass:main}. 

In settings with discrete IVs, there is again rich literature on CATE estimation using point estimates \cite{Syrgkanis.2019, Frauen.2023b}. There are also different efforts aimed at the \emph{derivation} of bounds \cite{swanson2018partial}, but methods for \emph{estimation}, such as meta-learners, are scarce. To the best of our knowledge, our work is the first to propose meta-learners for estimating such bounds.  

\section{Discussion}

\textbf{Conclusion:} In this paper, we derive tight bounds for partial identification of the CATE with data from multiple environments by showing the relation of our setting to IV settings. Further, we propose novel model-agnostic meta-learners for estimating these bounds from observational data. We find that the cross-environment bounds are especially helpful for learning tight bounds in areas where the treatment assignment policy varies clearly between the environments, and that the performance ranking of the respective meta-learners depends on the data generating process. Hence, our proposed meta-learners provide a valuable tool for tailoring the estimation of bounds for the CATE with data from multiple environments to the respective properties of different settings.

\textbf{Future directions:} We envision that our ideas could be used to develop meta-learners in other partial identification settings, for which effective estimation procedures are often lacking. Examples are meta-learners for settings with continuous instruments, leaky mediation, or sensitivity analysis. 

\clearpage

\section*{Impact statement} 

We acknowledge our meta-learners build on formal assumptions that are standard in the causal inference literature. Hence, we recommend a cautious and responsible use of our meta-learners to ensure that the assumptions are met, so that reliable inferences can be made. Even though we aim to relax standard assumptions such as unconfoundedness, we still implicitly assume that the environmental variable $E$ acts as an instrument, i.e., is not correlated with any potential unobserved confounders. This assumption can generally not be tested and must be justified by domain knowledge. However, we note that there are many real-world scenarios where this assumption holds, for example, RCTs with non-compliance \cite{Finkelstein.2012}.

Notwithstanding, our work aims to relax the (strict) assumptions of prior research by allowing for partial identification (instead of point identification). This can help to make causal inference more robust. As such, it can even help to improve the reliability of inferences for marginalized groups. For example, marginalized groups are different from the rest of the population in terms of their socio-economic status, which is a known confounder in medical studies, so that causal inference that ignores such confounder may be biased, while our meta-learners can help to make more reliable inferences.  

\bibliography{bibiography}
\bibliographystyle{icml2024}

%%%%%%%%%%%%%%%%%%%%%%%%%%%%%%%%%%%%%%%%%%%%%%%%%%%%%%%%%%%%%%%%%%%%%%%%%%%%%%%
%%%%%%%%%%%%%%%%%%%%%%%%%%%%%%%%%%%%%%%%%%%%%%%%%%%%%%%%%%%%%%%%%%%%%%%%%%%%%%%
% APPENDIX
%%%%%%%%%%%%%%%%%%%%%%%%%%%%%%%%%%%%%%%%%%%%%%%%%%%%%%%%%%%%%%%%%%%%%%%%%%%%%%%
%%%%%%%%%%%%%%%%%%%%%%%%%%%%%%%%%%%%%%%%%%%%%%%%%%%%%%%%%%%%%%%%%%%%%%%%%%%%%%%
\newpage
\appendix
\onecolumn

\section{Extended related work}
\label{app:rw}

\textbf{Causal inference using multiple datasets:} A related stream of literature focuses on causal inference from multiple datasets and can be roughly separated into works that leverage (i)~randomized or (ii)~purely observational data.

Works from (i) often leverage small amounts of randomized data in combination with observational data to obtain unbiased causal effect estimates. For example, \citet{Kallus.2018e} and \citet{Hatt.2022} use the observational data to obtain a biased CATE estimator, which is then subsequently debiased using the randomized data. Similar approaches exist for long-term effects \citep{Athey.2020, Ghassami.2022b, Imbens.2022}. An estimation theory for more general causal effects has been proposed by \citet{Jung.2023}. Finally, bounds for causal effects have been derived by \citet{Zhang.2022c} and \citet{Ilse.2023}. Note that all of these works require randomized data, while our paper relies on purely observational data.

Works from (ii) include invariant causal prediction methods \cite{Peters.2016, HeinzeDeml.2018, Mooij.2020} that aim to discover the causal predictors of a target variable. This idea has been adapted to CATE estimation to find a valid adjustment set of covariates \cite{Shi.2021}. \citet{Bareinboim.2016} proposed a theory for identifiability and transportability of causal effects across different environments under potential selection bias. Other works that leverage observational datasets from different environments include transfer learning for CATE under unconfoundedness \cite{Bica.2022} and methods for detecting unobserved confounding \cite{Karlsson.2023}. However, none of these works learns bounds for the CATE under violations of assumptions.

\textbf{Model-based CATE estimation:} Another stream of literature focuses on adapting specific machine learning models for CATE estimation. Examples include forest-based methods \cite{Wager.2018} and neural networks, which leverage shared representations across response surfaces and propensity scores to improve finite-sample performance \citep{Shi.2019, Curth.2021}. A complementary approach is to learn balanced representations that are unpredictable of the treatment \cite{Johansson.2016, Shalit.2017}. These methods are in contrast to model-agnostic meta-learners, which can be used with any \emph{arbitrary} machine learning model.

\clearpage

\section{Proofs}
\label{app:proofs}

\subsection{Proof of Lemma~\ref{lem:manski}}
\begin{proof}
The result follows from decomposing the oracle response function into a factual and counterfactual part, which can be bounded. That is,
\begin{align}
\widetilde{\mu}_{a}(x) &= \pi^e_a(x) \E[Y(a) \mid X = x, A = a] + (1 - \pi^e_a(x)) \E[Y(a) \mid X = x, A \neq a] \\
& = \pi^e_a(x) \mu^e_a(x) + (1 - \pi^e_a(x)) \E[Y(a) \mid X = x, A \neq a] \\
& \leq \pi^e_a(x) \mu^e_a(x) + (1 - \pi^e_a(x)) s_2,
\end{align}
where we used the definition of $s_1$ and $s_2$ as the support boundary of $Y$.
\end{proof}

\subsection{Proof of Theorem~\ref{thrm:bounds}}
\begin{proof}
Note that 
\begin{equation}
    \tau_{a_1, a_2}(x) = \widetilde{\mu}_{a_1}(x) - \widetilde{\mu}_{a_2}(x) .
\end{equation}
Hence, Lemma~\ref{lem:manski} and Assumptions~\ref{ass:consistency} and \ref{ass:main} imply that
\begin{equation}
 b^-_{e,j}(x)  \leq \tau_{a_1, a_2}(x) \leq b^+_{e,j}(x)
\end{equation}
for all environment $e$ and $j$. Hence, taking the minimum and maximum over $e$ and $j$ yields the result.
\end{proof}

\subsection{Proof of Theorem~\ref{thrm:consistency}}
\begin{proof}
Without loss of generality, we show the result for the upper bounds. For lower bounds, the proof works analogously by interchanging the support points $s_1$ and $s_2$. We proceed by calculating $\E[\hat{B}^+_{e, j} \mid X = x]$ for each pseudo-outcome $\hat{B}^+_{e, j}$, which corresponds to an oracle second stage regression. We start with the WB-learner, which uses pseudo-outcomes defined via
\begin{equation}
    \hat{B}^{+\mathrm{WB}}_{e} = \mathbbm{1}\{E=e\} \left(A Y + (1-A)s_2 - (1 - A)Y - A s_1 \right).
\end{equation}
Hence, we obtain
\begin{align}
    \E[\hat{B}^{+\mathrm{WB}}_{e} \mid X = x, E = e] &= \E[AY \mid X = x, E = e] + s_2 \E[1-A \mid X = x, E = e] \\ 
    &\quad- \E[(1-A)Y \mid X = x, E = e] - s_1 \E[A \mid X = x, E = e]\\
    & = \pi_1^e(x) \mu^e_{1}(x) + \pi_0^e(x) s_2  - \pi_0^e(x) \mu^e_{0}(x) - \pi_1^e(x) s_1 = b^+_{e,e}(x) .
\end{align}
as desired. The CB-RA-learner uses pseudo-outcomes defined via
\begin{align}
    \hat{B}^{+\mathrm{RA}}_{e, j} &= \mathbbm{1}\{E=e\}\left(\widetilde{Y}^+_{e,j} - \hat{r}^+_{j}(x) \right)  
    + \mathbbm{1}\{E=j\} \left(\hat{r}^+_{e}(x) -\widetilde{Y}^+_{e,j}\right) \\ &+ \mathbbm{1}\{E\neq e\} \mathbbm{1}\{E\neq j\} \left(\hat{r}^+_{e}(x) - \hat{r}^+_{j}(x) \right) .
\end{align}
Taking expectation conditional on $X = x$ yields
\begin{align}
    \E[\hat{B}^{+\mathrm{RA}}_{e, j} \mid X = x] &= \delta_e(x) \left(\E[AY \mid X = x, E = e] + s_2 \E[1-A \mid X = x, E = e] - \hat{r}^+_{j}(x)\right) \\
    &\quad + \delta_j(x) \left(\hat{r}^+_{e}(x) - \E[(1-A)Y \mid X = x, E = j] - s_1 \E[A \mid X = x, E = j] \right)\\
    & \quad + (1 - \delta_e(x) - \delta_j(x))\left( \hat{r}^+_{e}(x) - \hat{r}^+_{j}(x)\right)\\
    &= \delta_e(x) \left(\pi_1^e(x) \mu^e_{1}(x) + \pi_0^e(x) s_2 -\hat{r}^+_{j}(x) \right) \\
    &\quad + \delta_j(x) \left(\hat{r}^+_{e}(x) - \pi_0^j(x) \mu^j_{0}(x) - \pi_1^j(x) s_1 \right)\\
    &\quad + (1 -\delta_e(x) - \delta_j(x)) \left(\hat{r}^+_{e}(x) - \hat{r}^+_{j}(x) \right) .
\end{align}
Hence, whenever $\hat{r}^+_{e}(x) = {r}^+_{e}(x) = \pi_1^e(x) \mu^e_{1}(x) + \pi_0^e(x) s_2$ and $\hat{r}^+_{j}(x) = {r}^+_{j}(x) = \pi_0^j(x) \mu^j_{0}(x) - \pi_1^j(x) s_1 $, we obtain
\begin{equation}
\E[\hat{B}^{+\mathrm{RA}}_{e, j} \mid X = x] = \delta_e(x) b^+_{e,j}(x) + \delta_j(x) b^+_{e,j}(x) + + (1 - \delta_e(x) - \delta_j(x)) b^+_{e,j}(x) = b^+_{e,j}(x).
\end{equation}

The pseudo-outcomes of the CB-IPW-learner are defined via
\begin{equation}
    \hat{B}^{+\mathrm{IPW}}_{e, j} = \frac{\mathbbm{1}\{E=e\}}{\hat{\delta}_e(x)} \left(AY + (1-A)s_2\right)
     - \frac{\mathbbm{1}\{E=j\}}{\hat{\delta}_j(x)} \left((1-A)Y + A s_1\right),
\end{equation}
which yields
\begin{align}
    \E[\hat{B}^{+\mathrm{IPW}}_{e, j} \mid X = x] &= \frac{\delta_e(x)}{\hat{\delta}_e(x)} \left( \E[AY \mid X = x, E = e] + s_2 \E[1-A \mid X = x, E = e] \right) \\
    &\quad - \frac{\delta_j(x)}{\hat{\delta}_j(x)} \left(\E[(1-A)Y \mid X = x, E = j] + s_1 \E[A \mid X = x, E = j] \right) \\
    &= \frac{\delta_e(x)}{\hat{\delta}_e(x)} \left(\pi_1^e(x) \mu^e_{1}(x) + \pi_0^e(x) s_2 \right) - \frac{\delta_j(x)}{\hat{\delta}_j(x)} \left(\pi_0^j(x) \mu^j_{0}(x) + \pi_1^j(x) s_1 \right).
\end{align}
Hence, $\hat{\delta}_\ell(x) = \delta_\ell(x)$ implies
\begin{equation}
    \E[\hat{B}^{+\mathrm{IPW}}_{e, j} \mid X = x] = b^+_{e,j}(x).
\end{equation}

Finally, the pseudo outcomes of the CB-DR-learner are defined via
\begin{equation}
    \hat{B}^{+\mathrm{DR}}_{e, j} = \hat{B}^{+\mathrm{IPW}}_{e, j} + \left(1 - \frac{\mathbbm{1}\{E=e\}}{\hat{\delta}_e(x)}\right) \hat{r}^+_{e}(x) - \left(1 - \frac{\mathbbm{1}\{E=j\}}{\hat{\delta}_j(x)}\right) \hat{r}^+_{j}(x) .
\end{equation}
Again, by taking expectation conditional on $X = x$, we obtain
\begin{align}
    \E[\hat{B}^{+\mathrm{DR}}_{e, j} \mid X = x] 
    &= \frac{\delta_e(x)}{\hat{\delta}_e(x)} \left(\pi_1^e(x) \mu^e_{1}(x) + \pi_0^e(x) s_2 \right) - \frac{\delta_j(x)}{\hat{\delta}_j(x)} \left(\pi_0^j(x) \mu^j_{0}(x) + \pi_1^j(x) s_1 \right) \\
    & \quad + \left(1 - \frac{\delta_e(x)}{\hat{\delta}_e(x)} \right) \hat{r}^+_{e}(x)  - \left(1 - \frac{\delta_j(x)}{\hat{\delta}_j(x)} \right) \hat{r}^+_{j}(x) . \\
\end{align}
 Under $\hat{\delta}_\ell(x) = \delta_\ell(x)$, this reduces to
\begin{align}
    \E[\hat{B}^{+\mathrm{DR}}_{e, j} \mid X = x]  &= \left(\pi_1^e(x) \mu^e_{1}(x) + \pi_0^e(x) s_2 \right) - \left(\pi_0^j(x) \mu^j_{0}(x) + \pi_1^j(x) s_1 \right)  + 0 \, \hat{r}^+_{e}  - 0 \, \hat{r}^+_{j} \\
    &= b^+_{e,j}(x).
\end{align}
Under $\hat{r}^+_{e}(x) = {r}^+_{e}(x) = \pi_1^e(x) \mu^e_{1}(x) + \pi_0^e(x) s_2$ and $\hat{r}^+_{j}(x) = {r}^+_{j}(x) = \pi_0^j(x) \mu^j_{0}(x) - \pi_1^j(x) s_1 $, the expression reduces to
\begin{align}
    \E[\hat{B}^{+\mathrm{DR}}_{e, j} \mid X = x] 
    &= \left(1 + \frac{\delta_e(x)}{\hat{\delta}_e(x)} - \frac{\delta_e(x)}{\hat{\delta}_e(x)} \right) \left(\pi_1^e(x) \mu^e_{1}(x) + \pi_0^e(x) s_2 \right) - \left(1 + \frac{\delta_j(x)}{\hat{\delta}_j(x)} - \frac{\delta_j(x)}{\hat{\delta}_j(x)} \right)\left(\pi_0^j(x) \mu^j_{0}(x) + \pi_1^j(x) s_1 \right) \\ &= b^+_{e,j}(x),
\end{align}
which proves the result.
\end{proof}

\clearpage

\section{Additional Insights}
In Figures~\ref{fig:results_synthetic} and \ref{fig:results_covid}, we plot the learned bounds for (i) both within-environment bounds and (ii) both cross-environment bounds to demonstrate the general performance of our learners. However, in practice, we are usually only interested in the \emph{tightest} bounds as stated in Theorem~\ref{thrm:bounds}. Hence, for completeness, we plot the learned tightest bounds from Figures~\ref{fig:results_synthetic} and \ref{fig:results_covid} separately in Figures~\ref{fig:results_synthetic_tightest} and \ref{fig:results_covid_tightest}. We observe the same patterns for the tightest bounds as for all possible bounds: For the synthetic dataset, the CB-DR-learner results in valid and slightly more stable estimates with lower variation between runs compared to the na{\"i}ve plugin learner. 

For the real-world dataset, where we cannot compare with ground-truth bounds, both selected learners yield similar estimates. Also, in both considered settings, almost solely our proposed cross-environment bounds result in the tightest bounds, while the within-environment bounds are not helpful for yielding informative bounds. Here, this implies that one should focus on proper estimation of the CB-bounds, which can be done by our different proposed two-stage CB-meta-learners. In sum, this demonstrates that both contributions, (i) the bounds making use of data from different environments, and (ii) the meta-learners, are useful for learning informative bounds robustly.

\begin{figure}[ht]
  \centering
\begin{subfigure}
  \centering
  \includegraphics[height=0.24\linewidth]{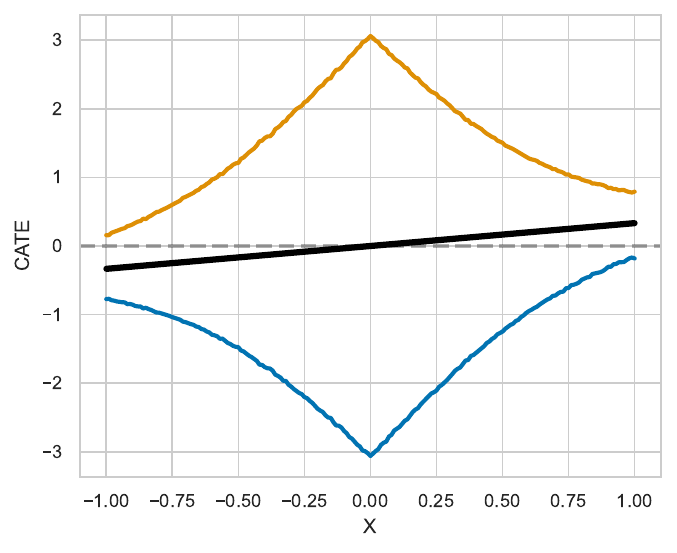}
\end{subfigure}%
\begin{subfigure}
  \centering
  \includegraphics[height=0.24\linewidth]{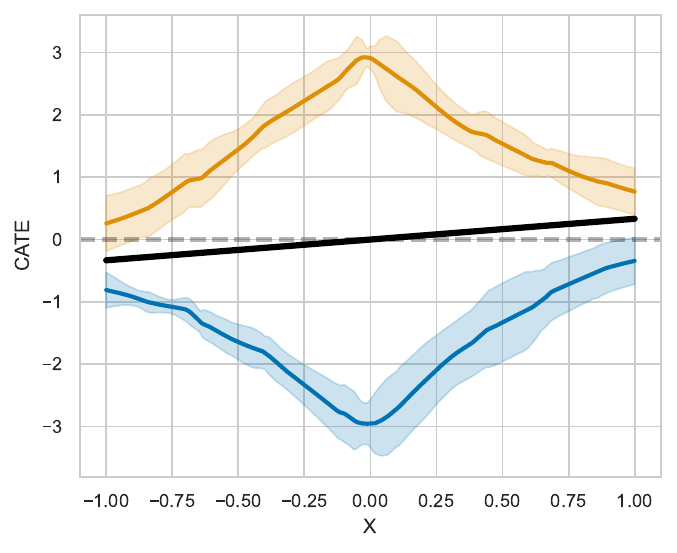}
\end{subfigure}
\begin{subfigure}
  \centering
  \includegraphics[height=0.24\linewidth]{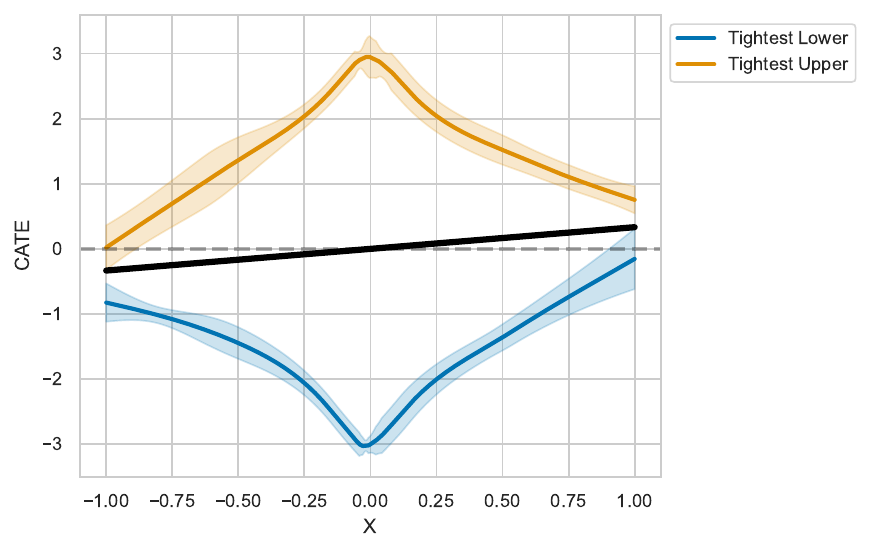}
\end{subfigure}
\vspace{-0.4cm}
\caption{Comparison of estimation methods for \emph{tightest} predicted bounds based on synthetic dataset 2 (predicted bounds $mean \pm 3std$ over 5 runs). \textbf{Left:} Oracle bounds for within- (WB) and cross- (CB) environments. \textbf{Center:} Estimated bounds by the na{\"i}ve plug-in learner. \textbf{Right:} Estimated bounds by our two-stage meta-learners (here: WB-learner with CB-DR-learner).}
\label{fig:results_synthetic_tightest}
\vspace{-0.2cm}
\end{figure}

\begin{figure}[ht]
  \centering
\begin{subfigure}
  \centering
  \includegraphics[width=0.3\linewidth]{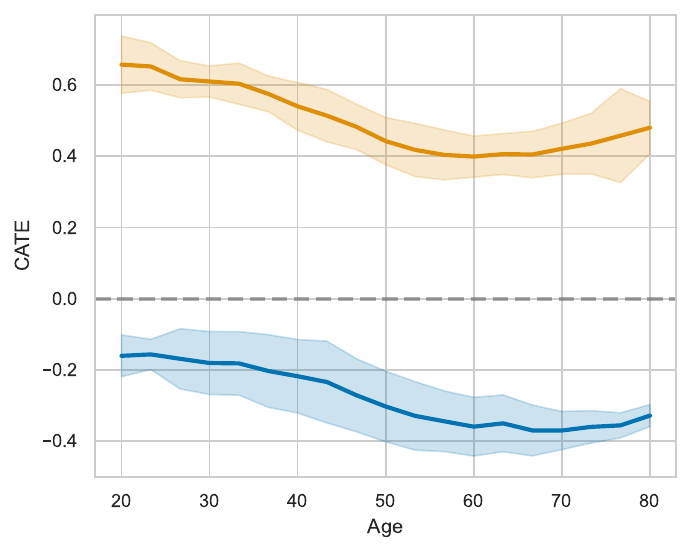}
\end{subfigure}%
\begin{subfigure}
  \centering
  \includegraphics[width=0.3\linewidth]{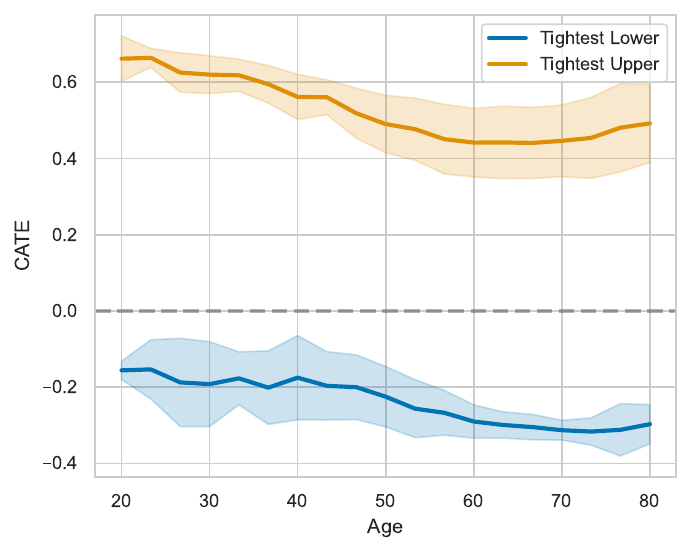}
\end{subfigure}
\vspace{-0.4cm}
\caption{Insights from real-world data: Effect of comorbidity on mortality in COVID-19 patients in Brazil. \textbf{Left:} Estimated \emph{tightest} bounds by na{\"i}ve plug-in learner. \textbf{Right:} Estimated \emph{tightest} bounds by our two-stage meta-learners (here: WB-learner and CB-PI-learner).}
\label{fig:results_covid_tightest}
\vspace{-0.2cm}
\end{figure}

\clearpage

\section{Implementation details}\label{app:implementation}

\textbf{Model architecture and parameters:} For our experiments, we adapt the implementations used in previous works for evaluating meta-learners for point-identified CATE estimation \cite{Curth.2021, Curth.2021c}. Importantly, we use similar network architectures and parameter settings for all meta-learners. Thus, performance differences are decoupled from model choice and can be explained by the meta-learners as the source of gain. In detail, we use the software package \url{https://github.com/AliciaCurth/CATENets} and all of the default settings of the PyTorch CATE meta-learners provided in this package. Here, the networks for the first- and second-stage models are simple MLPs with 2 hidden layers and hidden neuron size of 100. For the nuisance function estimation of our na{\"i}ve plugin estimator, we implement similar architectures.

\textbf{Implementation:} By using our pseudo-outcomes for bound estimation as described in Sec.~\ref{sec:two-stage-learners}, we can learn the bounds as follows: For the CB-learners, we define the pseudo-outcomes of Eq.~\eqref{eq:cb-pseudo-outcome} for each possible environment combination $e, j$ and for upper and lower bounds, respectively. Then, by simply replacing the factual outcomes with the pseudo-outcomes, and the treatment assignment with the environment assignment, we can train the respective two-stage-learners by using the above-mentioned implementation of the CATE meta-learners. We train an own meta-learner for each environment combination and each upper and lower bound. For the WB-learner, we simply directly learn models to predict the pseudo-outcomes of Eq.~\eqref{eq:wb_learner} for each environment $e$, and upper and lower bound, respectively. We provide a wrapper module for our bound meta-learners in our code.\footnote{\url{https://github.com/JSchweisthal/BoundMetaLearners}}

\clearpage

\section{Details regarding simulated data}\label{app:sim}

\textbf{Data-generating process:} For our synthetic dataset 1, we simulate an observed confounder $X \sim \mathrm{Uniform}[-1, 1]$ and an unobserved confounder $U \sim \mathrm{Uniform}[0, 1]$. We  define $\delta(x) = \mathbb{P}(E=1 \mid x) = \sigma(x)$, where $\sigma(\cdot)$ denotes the sigmoid function $\sigma(x) = \frac{1}{1+exp(-x)}$. Then, we simulate the environment $E$ via
\begin{equation}
E= 1 \mid X = x\sim \textrm{Bernoulli}\left(p= \delta(x))\right).
\end{equation}
We then sample our treatment assignments from the environment-specific propensity scores as 
\begin{equation}
A= 1 \mid X = x, U = u, E=1 \sim \textrm{Bernoulli}\left(p= \sigma\left(2.5x + u\right) \right),
\end{equation} 
and
\begin{equation}
A= 1 \mid X = x, U = u, E=0 \sim \textrm{Bernoulli}\left(p= \left(1- \sigma\left(2.5x + u\right) \right) \right) .
\end{equation}

We define the treatment effect as 
\begin{equation}
\tau(x) = \frac{1}{3} x .
\end{equation}

Finally, we simulate a continuous outcome 
\begin{equation}
Y = \tau(X) A +  \frac{1}{3}(\text{sin}(12X)+X) + \frac{1}{60} \text{ cos}(2X)  + U + 0.3\varepsilon,
\end{equation}
where $\varepsilon \sim \text{Laplace}(0, 1)$.

For our synthetic dataset 2, we model a more complex environment probability. Here, we model $\delta(x) = 0.15 \, \text{sin}(5x)+0.5$, while keeping the remaining process unchanged, such that we can isolate and show the effect of the environment probability on the performance of the different meta-learners.

To create the simulated data used in Sec.~\ref{sec:experiments}, for both datasets, we sample $n=10000$ from the data-generating process above. We then split the data into train (70\%), val (10\%), and test (20\%) sets.

\clearpage

\section{Details regarding real-world data}\label{app:real}
 We perform a case study using a dataset with COVID-19 hospitalizations in Brazil across different regions \cite{baqui2020ethnic}. We are interested in predicting the effect of comorbidity on the mortality of COVID-19 patients. For the environments, we use the regions of the hospitals in Brazil, which are split into North and Central-South. As observed confounders, we include age, sex, and ethnicity. Further, we exclude patients younger than 20 or older than 80 years. To define comorbidity as a binary variable, we define comorbidity as 1 if at least one of the following conditions were diagnosed for the patient: cardiovascular diseases, asthma, diabetes, pulmonary disease, immunosuppression,
       obesity, liver diseases, neurological disorders, renal disease.
 We then split the
data into train (70\%), val (10\%), and test (20\%) sets. For training, we use the same settings as reported in Appendix~\ref{app:implementation}.

\end{document}